\documentclass{article}
\usepackage{spconf,amsmath,graphicx}
\usepackage{graphicx}
\usepackage{epstopdf}
\usepackage{amsfonts,amssymb,amsmath,amsthm}
\usepackage{float}
\usepackage[footnotesize]{caption}

\newtheorem{proposition}{Proposition}
\newtheorem{definition}{Definition}
\newcommand{\bs}{\boldsymbol}
\newcommand{\cl}{\mathcal}
\newcommand{\ie}{\emph{i.e.}, }
\newcommand{\eg}{\emph{e.g.}, }
\renewcommand{\Re}{\mathbb R}
\newcommand{\sm}[2]{
\big(\begin{smallmatrix}
  #1\\
  #2
\end{smallmatrix}\big)
}

\title{\ \\[-9mm] Multi-resolution Compressive Sensing Reconstruction\vspace{-3mm}}
\name{Adriana Gonz{\'a}lez$^1$,  Hong Jiang$^2$, Gang Huang$^2$, Laurent Jacques$^1$\vspace{-3mm}
\thanks{LJ is funded by the F.R.S.-FNRS. The authors would like to thank Patrice Rondao Alface of Nokia Bell Labs for his interest and insightful discussions.}}
\address{
\ninept $^1$ICTEAM/ELEN, ISPGroup, UCL, Belgium.
\ninept $^2$ Bell Labs, NJ, USA.\vspace{-3mm}
}

\newcommand{\sq}{\vspace{-2mm}}
\newcommand{\sqm}{\vspace{-.7mm}}

\begin{document}
%
\maketitle
\begin{abstract}
We consider the problem of reconstructing an image from compressive measurements using a multi-resolution grid. In this context, the reconstructed image is divided into multiple regions, each one with a different resolution. This problem arises in situations where the image to reconstruct contains a certain region of interest (RoI) that is more important than the rest. 
Through a theoretical analysis and simulation experiments we show that the multi-resolution reconstruction provides a higher quality of the RoI compared to the traditional single-resolution approach. 
\end{abstract}

\sq
\section{Introduction}
\label{sec:intro}
\sq

Compressive sensing (CS)~\cite{CRT2005, Donoho, candes2006} is a
powerful technique for efficiently acquiring images and
videos~\cite{romberg2008, Li, jiang2012} at a sampling rate depending
on their intrinsic complexity (\eg their sparsity). For instance, in
image acquisition \cite{lustig2008, takhar2006, huang2013}, an
informative image can thus be captured in terms of compressive
measurements, \ie using far fewer measurements than the number of
pixels in the image. This makes the acquisition more cost-effective
and less bandwidth-demanding. Efficient recovery algorithms have been developed to reconstruct an image from a set of compressive measurements~\cite{CRT2005, jiang2012, jalali2014, metzler2014}. These algorithms reconstruct an image from the measurements using the same resolution (number of pixels) as that of the original image that generated the measurements. 

The pixels of the original image generating the measurements usually form a uniform grid. The measurements from a compressive imaging device, such as those described in~\cite{takhar2006, huang2013}, are acquired using a mask consisting of an array of programmable elements where each element defines a pixel of the image and the size of the elements is the same. However, oftentimes, an image has a region of interest (RoI) that contains more details than the rest of the image. In those cases, it is desirable to have a higher resolution (more pixels) in the RoI than in other regions so the details of the RoI can be better resolved. For example, in a portrait, the person's face has often more details than the background, which may be blurred. In images where a refocus is applied after acquisition~\cite{ng2005}, it would be desirable for the region of refocus to have a higher resolution than the rest of the image. Another example can be found in the reconstruction of an image from measurements made by a lensless compressive imaging device with multiple sensors, where the region of interest is the intersection of all the views of the multiple sensors~\cite{JHW2014}.
 
In this paper, we aim at reconstructing, from a given set of compressive measurements, a high quality version of a certain RoI inside an image. The measurements are assumed to be generated from an image with a uniform grid of pixels. Our approach consists of dividing the reconstruction grid into regions with different resolutions, reserving a higher resolution for the RoI. More specifically, in this paper, the RoI is described using a grid with the same resolution as that of the original image, and the rest of the image is described using a lower resolution grid. As an example, we consider the image shown in Fig.~\ref{fig01}-(left). We identify the RoI as the region containing the butterfly on the right hand side of the image. Outside of the RoI, the pixels are relatively constant and can be well represented in a lower resolution than that in the RoI. The pixel distribution in the multi-resolution reconstruction is illustrated in Fig.~\ref{fig01}-(right).

 \begin{figure}[ht]
 	\centering
	\begin{tabular}{c c}
	\vspace{-1cm} & \\
 	\hspace{-0.4cm} \includegraphics[height=2.5cm]{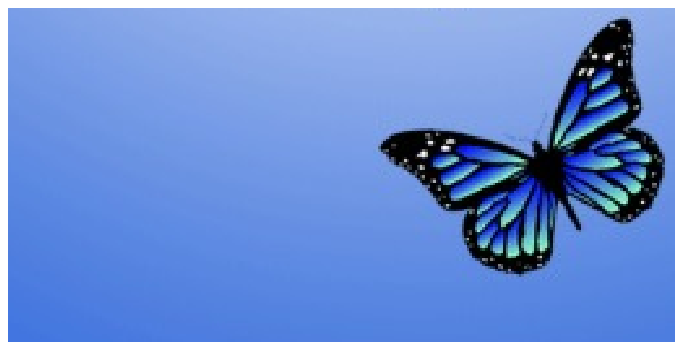} \hspace{-1.2cm} &
 	\includegraphics[height=2.67cm]{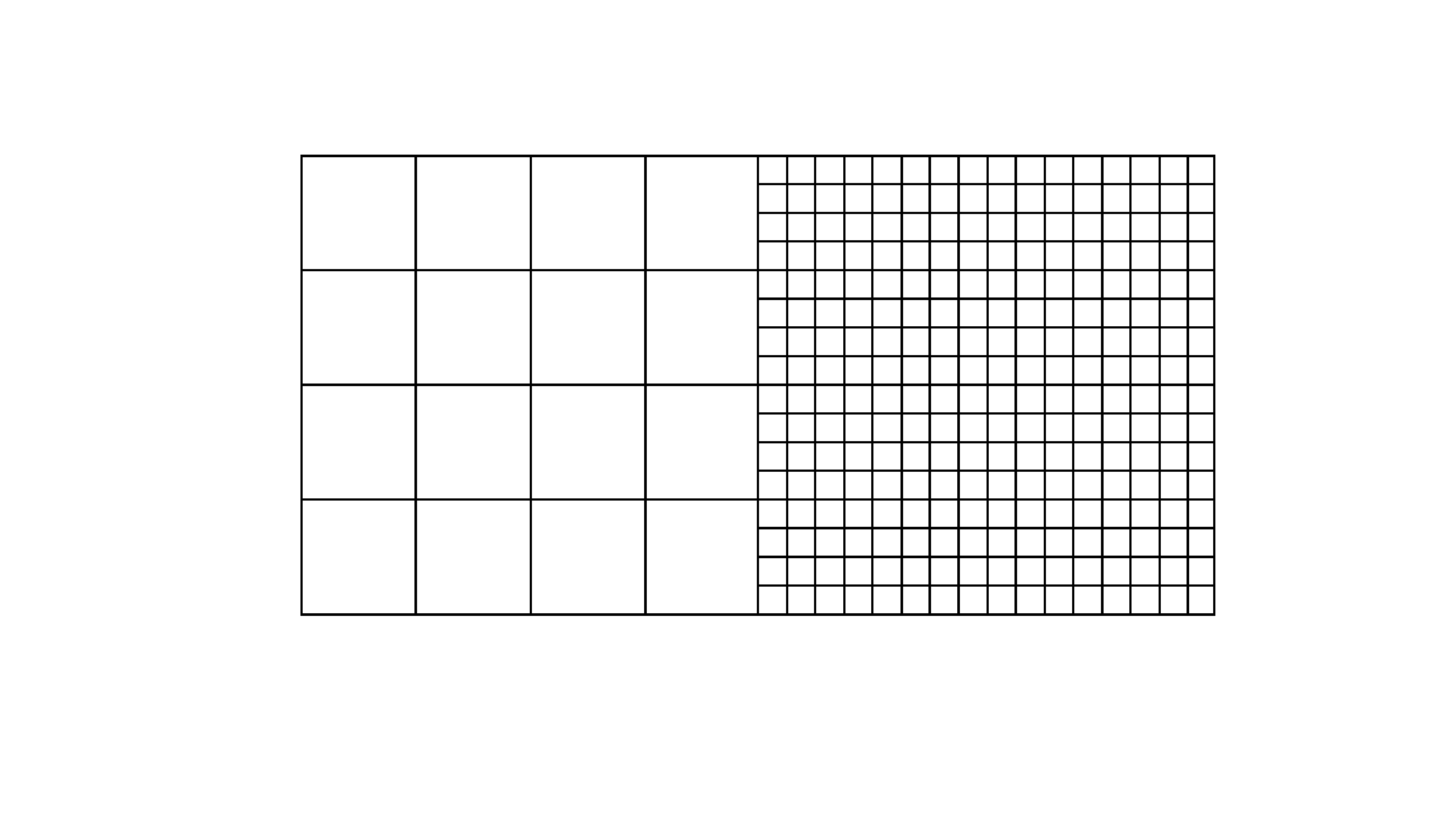} \\
	\vspace{-1cm} & \\
	\end{tabular}
 	\caption{An image with region of interest (RoI). Left: Original image. Right: Illustration of a multi-resolution grid with lower resolution outside of the RoI.\\[-1cm]}
 	\label{fig01}
 \end{figure}
\vspace{0.5cm}

The motivation for this approach is largely based on the CS theory stating that, for a given image, the quality of the reconstructed image depends on the ratio of the number of measurements over the number of pixels. In general, a larger ratio implies a better quality of the reconstructed image. Therefore, the multi-resolution approach aims at reducing the total number of pixels in the reconstructed image by decreasing the resolution outside of the RoI. This effectively increases the number of measurements as compared to the total number of pixels. Consequently, we expect an improved quality of the RoI in the reconstructed image. 

\sq
\subsection{Main contribution of this paper}
\sq

We formulate the multi-resolution reconstruction as a constrained minimization problem in which the unknowns are the (non-uniform) pixels of the multi-resolution image to be reconstructed. The constraint is determined by the available measurements. Since the measurements are generated from an image described in a uniform grid, an upsampling of the lower resolution regions is necessary before the sensing matrix is applied. 

On one hand, since the number of unknowns is reduced as compared to the single-resolution reconstruction, while the number of equations remains the same, it is expected that the multi-resolution approach would provide an RoI of higher quality. On the other hand, the upsampling operation may introduce an error into the optimization problem, potentially affecting the reconstruction quality. What is the overall quality of the multi-resolution reconstruction? Is there a gain over the traditional single-resolution reconstruction? 

These questions are addressed in the following by a theoretical analysis and simulation experiments. A theoretical bound is also provided on the error of the multi-resolution reconstruction, which sheds a light on the conditions under which the multi-resolution approach is more advantageous. 

\sq
\subsection{Related work}
\sq
Multi-resolution for CS have been studied previously in~\cite{jiang2010, jiangBLTJ2012, goldstein2013, zhu2014, wang2015}. These works have focused on the reconstruction of an image using a uniform grid with a different resolution than the one used to generate the measurements. This is useful, for example, in a broadcast system where the same measurements may be used by different receivers to reconstruct images of different resolutions that fit their own needs. In contrast to these works, where the multi-resolution schemes use a uniform grid for the reconstruction, the multi-resolution approach proposed in this paper refers to the reconstruction using pixels distributed in a non-uniform grid: different regions in the image may have a different resolution.   

\sq
\subsection{Organization of the paper}
\sq
In Section~\ref{sec:MR}, we present the mathematical formulation of the multi-resolution reconstruction problem and describe the main result on the error bound of the multi-resolution reconstruction. In Section~\ref{sec:simulations}, simulation results are presented in order to support the theoretical analysis. We end the paper with a conclusion in Section~\ref{sec:conc}.

\sq
\section{Theoretical analysis}
\label{sec:MR}
\sq

In this section, we formally formulate the problem and provide a theoretical analysis for the multi-resolution reconstruction. Although the same idea and analysis applies to multiple regions with more than two resolutions, we limit our discussion here to two regions where the RoI has the same resolution as the original image, and the rest of the image has a lower resolution.

\sq
\subsection{Problem formulation}
\sq 

For the sake of simplicity, we consider hereafter that the image and
its restrictions to both the RoI and to its complement are sparse in
the canonical (pixel) basis. Consider the compressive sensing
reconstruction problem
\sqm
\begin{equation}
\mathop {\min }\limits_{\bs x \in {\Re ^N}} { \| \bs x \|_1 } \ {\rm s. t.} \ \left\| {\bs y - \bs A \bs x} \right\|_2 \le \varepsilon,\sqm
\label{eq:min1}
\end{equation}
 where $\bs A \in {\Re^{m \times N}}$ is a sensing matrix, $\bs y \in {\Re^m}$ is a measurement vector, and $\varepsilon\geq0$ is a bound of the noise in the measurements. The resolution of a solution $\widetilde{\bs x}$ provided by~\eqref{eq:min1} is $N$.

Let us consider that the image $\bs x$ is divided into two disjoint
regions. Formally, with $[N] :=\{1,\cdots, N\}$, we consider $\cl
J^1 \subset [N]$ and $\cl J^2 := [N] \setminus \cl J^2$, with $|\cl
J^1|=N_1$ and $|\cl J^2|=N_2=N-N_1$. The
vector $\bs x$ is then divided into two parts $\bs x_1$ and $\bs x_2$ with supports being ${\cl J^1}$
and ${\cl J^2}$, respectively, \ie 
\sqm
\begin{equation*}
	\bs x = \bs x_1 + \bs x_2,\quad \
	{\rm supp}\,\bs x_1 = \cl J^1,\ {\rm supp}\,\bs x_2 = \cl J^2. \sqm
\end{equation*}
With a slight abuse of notation, depending on the context, ${\bs x_1}$ and ${\bs x_2}$ will
either denote vectors living in ${\Re^{{N_1}}}$ and in ${\Re^{{N_2}}}$,
respectively, or vectors living in $\Re^N$ and equal to zero outside
of their supports. Up to a reordering of the columns of $\bs A$, we
can always form
\sqm
\begin{equation*}
  \bs A = \left[{\bs A_1},{\bs A_2}\right] \in {\Re ^{m \times ({N_1}
      + {N_2})}},\ \bs x = \sm{\bs x_1}{\bs x_2}\in {\Re ^{({N_1} + {N_2})}}. \sqm	
\end{equation*}
We consider the region of interest (RoI) to be the area in the image
$\bs x$ corresponding to the indices in ${\cl J^1}$. Hereafter, the
resolution of the RoI is assumed to be the same as the one of the
original image, \ie $N_1$, and the rest of the image, corresponding to
${\cl J^2}$, has a lower resolution denoted ${N_L}$, with  
\sqm
\begin{equation*} 
N_L< {N_2}. \sqm
\end{equation*} 

Let $\bs E \in {\Re ^{{N_2} \times {N_L}}}$ be an expansion matrix
\cite{jiangBLTJ2012} mapping the image outside of the RoI from the
lower resolution grid in $\Re^{N_L}$ to the original resolution grid
in $\Re^{N_2}$. The expansion matrix $\bs E$ could, for example,
implement an interpolation of pixels. We define the multi-resolution
matrix ${\bs A_E}$ as 
\sqm
\begin{equation}
\begin{array}{l}
{\bs A_E} = [{\bs A_1},{\bs A_2} \bs E] \in {\Re ^{m \times N'}}, \ N'={N_1} + {N_L},\\[1mm]
{\bs A_E}\sm{\bs x_1}{\bs x_L} := {\bs A_1}{\bs x_1} + {\bs A_2}\bs E{\bs x_L}.
\end{array}\sqm
\label{eq:AE}
\end{equation}
The multi-resolution reconstruction problem can be formulated as:
\sqm
\begin{equation}
\mathop {\min }\limits_{\scriptstyle{\bs x_1} \in {\Re ^{{N_1}}}\hfill\atop
	\scriptstyle{\bs x_L} \in {\Re ^{{N_L}}}\hfill} { \| \sm{\bs
          x_1}{\bs x_L} \|_1 }\ {\rm s. t.} \ \left\| {\bs y - \bs A_E \sm{\bs
          x_1}{\bs x_L}} \right\|_2 \le \varepsilon. \sqm
\label{eq:min4}
\end{equation}
Problem~\eqref{eq:min4} is multi-resolution because it is optimized over regions ${\bs x_1}$ and ${\bs x_L}$ that have different resolutions. The question is how the solutions of \eqref{eq:min1} and \eqref{eq:min4} are related.

\sq
\subsection{Analysis}
\sq

We now establish a key property of the solution of \eqref{eq:min4}. 

\begin{definition}
\label{def:def1}

We say that $\bs A$ admits robust $k{\textrm{-sparse}}$ solutions if
there is a constant $C>0$ such that, for any $k$-sparse vector ${\bs
  x^*}$ with $\bs y = \bs A \bs x^* + \bs e$ where ${\left\| \bs e
  \right\|_2} \le \varepsilon$, every solution $\widetilde{\bs x}$
of~\eqref{eq:min1} satisfies the following estimate
\sqm
\begin{equation*}
	{\left\| {\widetilde{\bs x} - \bs x^*} \right\|_2} \le C\varepsilon. \sqm
\end{equation*}
\end{definition}

In~\cite{CRT2005,candes2006}, we can find sufficient conditions on
$\bs A$ such that~\eqref{eq:min1} admits robust solutions. For
example, if the entries of $\bs A$ are Gaussian random variables $\cl
N(0,1/m)$, then $\bs A$ admits robust $k{\textrm{-sparse}}$ solutions
with overwhelming probability if 
\sqm
\begin{equation}
m \geq m^*(k,N):=C \cdot k \cdot \log(N). \sqm
\label{eq:RIP}
\end{equation}

\sqm
\begin{proposition} 
\label{th:th2}
Let $\bs x^* = \bs x_1^* + \bs x_2^*$ be a $k{\textrm{-sparse}}$
vector and $\bs y = \bs A \bs x^*$. 
If the matrix $\bs A_E$ of~\eqref{eq:AE} admits robust $k{\textrm{-sparse}}$ solutions, then any  solution $\sm{\widetilde{\bs x}_1}{\widetilde{\bs x}_L}$ to
problem~\eqref{eq:min4} with $\varepsilon$ set to 
\sqm
\begin{equation*}
	\varepsilon_L  := \mathop {\min }\limits_{\bs x_L} \left\{ {{{\left\| {\bs A_2\left( {\bs E \bs x_L - \bs x_2^*} \right)} \right\|}_2}\;\left| {\;{{\left\| \bs x_L \right\|}_0} \le {{\left\| \bs x_2^* \right\|}_0}} \right.} \right\}\sqm
\end{equation*}
\sqm
satisfies
\sqm
\begin{equation}
	\left\| \widetilde{\bs x}_1 - \bs x_1^* \right\|_2 \le C \varepsilon_L. \sqm
\label{eq:15}
\end{equation}
\end{proposition}

Notice that $\varepsilon_L$ represents the error achieved by the best $\|\bs x_2^*\|_0$-sparse approximation of $\bs x_2^*$ through the expansion operator $\bs E$, as measured in the domain projected by $\bs A_2$.
\begin{proof} Let us define
\sqm
\begin{equation*}
	\bs x_L^* := \mathop {\arg \min }\limits_{\|\bs x_L\|_0 \le \|\bs x_2^*\|_0} {\left\| {\bs A_2 \left( \bs E \bs x_L - \bs x_2^* \right)} \right\|_2}\sqm
\end{equation*}
and $\bs e  := \bs A_2 \left( \bs x_2^* - \bs E \bs x_L^*
\right)$ with $\|\bs e\|_2 = \varepsilon_L$ by construction. 

The vector $\sm{\bs x_1^*}{\bs x_L^*}$ is $k$-sparse since
$\|\bs x_L^*\|_0 \le \|\bs x_2^*\|_0$ and 
\sqm
\begin{equation*}
	\bs A_E \sm{\bs x_1^*}{\bs x_L^*} = \bs A_1 \bs
        x_1^* + \bs A_2 \bs E \bs x_L^* = \bs y - \bs e.
\label{eq:proof2_2}\sqm
\end{equation*}
Therefore, if we set $\varepsilon = \varepsilon_L$ in
problem~\eqref{eq:min4}, $\sm{\bs x_1^*}{\bs x_L^*}$ is a feasible
$k$-sparse vector of its fidelity constraint. Consequently, if $\bs A_E$ admits robust $k$-sparse solutions, all solutions $\sm{\widetilde{\bs x}_1}{\widetilde{\bs x}_L}$ of~\eqref{eq:min4} satisfy
\sqm
\begin{equation*}
	\left\| \sm{\widetilde{\bs x}_1}{\widetilde{\bs x}_L} - \sm{\bs x_1^*}{\bs x_L^*} \right\|_2 \le C \varepsilon_L,\sqm
\end{equation*}
which leads to
\sqm
\begin{equation*}
		\left\| \widetilde{\bs x}_1 - \bs x_1^* \right\|_2 \le \left\| \sm{\widetilde{\bs x}_1}{\widetilde{\bs x}_L} - \sm{\bs x_1^*}{\bs x_L^*} \right\|_2
 		\le C \varepsilon_L. \vspace{-8mm}
\end{equation*}
\sq
\end{proof}
We note that Prop.~\ref{th:th2} only provides a bound on the solution $\widetilde{\bs x}_1$ in the RoI, but it does not provide a bound on the behavior of the low resolution portion $\widetilde{\bs x}_L$ outside of RoI. This is, however, sufficient because we are only interested in the quality of the image in the RoI.

\sq
\subsection{Discussion}
\label{sec:discussion_sec2}
\sq
Hereafter, we consider that the vector $\bs x^*$ in Prop.~\ref{th:th2} is the original image, $ \widetilde{\bs x}$ is the reconstructed image from the multi-resolution reconstruction~\eqref{eq:min4} and $\widetilde{\bs x}_1$ is the portion of the reconstructed image in the RoI.

We now assume that $\bs A \in {\Re ^{m \times N}}$ is a matrix whose
entries are identically and independently distributed (\emph{i.i.d.})
as ${\cl N}\left( 0,\frac{1}{m} \right)$~\cite{CRT2005,
  candes2006}. Hereafter, we denote by $\bs {\rm Id}_d$ the $d\times d$
identity matrix for $d\in \mathbb N$.

\begin{proposition}
If the entries of $\bs A$ are \emph{i.i.d.} as $\cl N(0,1/m)$ and if $\bs E^T \bs E = \bs {\rm Id}_{N_L}$, then the entries of $\bs A_E$ are also \emph{i.i.d.} as $\cl N(0,1/m)$.\sq
\end{proposition}
\begin{proof}
According to the decomposition \eqref{eq:AE}, we have trivially $(\bs
A_1)_{ij} \sim_{\rm iid} \cl N(0,1/m)$. Moreover, the entries of $\bs A_2' = \bs A_2 \bs E$ have
zero expectation and are clearly independent of those of $\bs
A_1$. Let us denote by $\bs a^T_i$ and $\bs b_j$ the
$i^{\rm th}$ row of $\bs A_2$ and the $j^{\rm th}$ column of $\bs E$,
respectively. Then, using the Kronecker symbol $\delta_{ij}$ equal to 1 if
$i=j$ and 0 otherwise, using $\bs b^T_i \bs b_j = (\bs E^T \bs E)_{ij}
= \delta_{ij}$, 
we find $\mathbb E (\bs A_2')_{ij} (\bs A_2')_{kl} = \mathbb E (\bs b^T_j \bs
a_i)(\bs a^T_k \bs b_l) = \bs b^T_j\,(\mathbb E \bs
a_i \bs a^T_k)\,\bs b_l =  m^{-1}\delta_{ik}\,\bs b_j^T \bs{\rm Id}_{N_2}\bs
b_l = m^{-1} \delta_{ik}\delta_{jl}$, which proves the independence of
$\bs A'_2$'s Gaussian random entries and concludes the proof.\sq  
\end{proof}

There are many situations where $\bs E^T \bs E = \bs {\rm Id}_{N_L}$. For
instance, let us consider that the region outside of the RoI is
compressible in the first $N_L$ elements of an orthonormal basis $\bs W_2 \in {\Re ^{N_2 \times
    N_2}}$, \eg in the first $N_L$ frequencies of a Discrete Cosine
Transform (DCT). This means that the vector $\bs x_2$ can be written as $\bs x_2 = \bs
W_2 \bs \alpha_2 $, with $\bs \alpha_2 \in \Re^{N_2}$ concentrated on
the first $N_L$ components. 

Under this assumption, the expansion matrix $\bs E$ can be formulated
as $\bs E = \bs W_2 \bs S_L^T \in \Re^{N_2 \times N_L}$, where $\bs
S_L \in \Re^{N_L \times N_2}$ is the operator selecting the $N_L$
lower frequency coefficients. It is then easy to check that $\bs E^T
\bs E = \bs {\rm Id}_{N_L}$ knowing that $\bs S_L \bs S^T_L = \bs {\rm Id}_{L}$ and $\bs W_2 \bs W^T_2=\bs W^T_2 \bs W_2 = \bs {\rm Id}_{N_2}$. 

Therefore, if $\bs A_E$ is generated as in \eqref{eq:AE} from a Gaussian random matrix
$\bs A$ and if $\bs E^T
\bs E = \bs {\rm Id}_{N_L}$, $\bs A_E$ admits robust
$k{\textrm{-sparse}}$ solutions with overwhelming probability
\cite{CRT2005,candes2006} if
\sqm 
\begin{equation}
m \geq m'(k,N'):= C' \cdot k \cdot \log(N'), ~~ N'=N_1+N_L . \sqm
\label{eq:RIPL}
\end{equation}
Since $N=N_1+N_2>N'=N_1+N_L$, a comparison of~\eqref{eq:RIP} and \eqref{eq:RIPL} shows that, asymptotically, $m'(k,N')<m^*(k,N)$. This implies that, for a given number of measurements $m$, the condition \eqref{eq:RIPL} may be satisfied while the condition~\eqref{eq:RIP} may not. 

Let us consider the case where $m$ is such that 
\sqm
\begin{equation}
m'(k,N') \leq m < m^*(k,N). \sqm
\label{eq:mcond}
\end{equation} 
Then, the result of Prop.~\ref{th:th2} holds and, in particular, the
error bound \eqref{eq:15} holds with overwhelming probability. In the
special case when $\bs E$ can be found to satisfy
\sqm
\begin{equation}
 \bs E \bs x^*_L =\bs x_2^*,\sqm
\label{eq:exact}
\end{equation}
Eq.~\eqref{eq:15} implies $\widetilde {\bs x}_1= \bs x^*_1$, \ie the multi-resolution reconstruction \eqref{eq:min4} produces the exact solution in the RoI. On the other hand, since~\eqref{eq:RIP} is not satisfied under the assumption of~\eqref{eq:mcond}, there is nothing that can be said about the solution of the single resolution reconstruction \eqref{eq:min1}. Hence, the image reconstructed using~\eqref{eq:min1} may be quite different from the original image $\bs x^*$. This is, therefore, the theoretical basis for the statement that the multi-resolution reconstruction has a higher quality in the RoI because the total number of unknowns is reduced. 

In general, we may not be able to find an expansion matrix $\bs E$ to satisfy \eqref{eq:exact} and, therefore, the term on right hand side of \eqref{eq:15} may not be zero. Consequently, the multi-resolution reconstruction may contain an error introduced by a potential mismatch in the mapping from the lower resolution to the higher resolution outside of the RoI. However, if the image in the region outside of the RoI is smooth, an expansion matrix can be found so that the right hand side of \eqref{eq:15} is very small, resulting in a high quality image in the RoI.

\sq
\section{Simulation}
\label{sec:simulations}
\sq

In this section, the multi-resolution approach~\eqref{eq:min4} is compared to the traditional single-resolution approach~\eqref{eq:min1} using compressive measurements. For this analysis, we use the image from Fig.~\ref{fig01}-(left), which is defined on a $128 \times 256$ pixel grid ($N = 32768$). The RoI is defined as the region containing the butterfly, which corresponds to the right half of the image, \ie $N_1 = N_2 = N/2$. 
Outside of the RoI, the pixel grid has a lower resolution defined through the parameter $L$ such that $N_L = N_2/L^2$. That is, the number of pixels outside of RoI is reduced by 4 and 16 times for $L=2$ and $L=4$, respectively. 

In the experiments, the sensing matrix $\bs A$ corresponds to a
Gaussian matrix with zero mean and variance $1/m$. The expansion
matrix $\bs E \in \mathbb R^{N\times N_L}$ corresponds to an interpolation matrix. Since the image in Fig.~\ref{fig01}-(left) is not sparse but compressible, the reconstructions in~\eqref{eq:min1} and~\eqref{eq:min4} were solved using an \emph{analysis} formulation with a redundant Haar wavelet basis as the sparsifying basis. We remark that this formulation does not exactly match the reconstruction problems in~\eqref{eq:min1} and~\eqref{eq:min4}, however, it allows illustrating the performance of both single and multi-resolution approaches on a compressible image and providing some intuition on the theoretical analysis presented in Section~\ref{sec:MR}.
The resulting minimization problems are solved in their constrained form using $\varepsilon = 10^{-10}$. The constraints are handled through the convex indicator functions on the sets \mbox{$\cl C_1 =~\{ \bs u \in \Re^N : \| \bs y - \bs A \bs u \|_2 \leq \varepsilon \}$} and \mbox{$\cl C_2 = \{ \sm{\bs u_1}{\bs u_L} \in~\Re^{N'} : \| \bs y - \bs A_E \sm{\bs u_1}{\bs u_L} \|_2 \leq \varepsilon \}$}, respectively. The resulting problems are solved using the algorithm proposed by Chambolle and Pock~\cite{CP2011}. 

The single and multi-resolution approaches are compared using the Reconstruction Signal-to-Noise Ratio (RSNR), defined as $\textrm{RSNR} = 20 \log_{10} \| \bs x^* \|_2/\| \widetilde{\bs x} - \bs x^* \|_2,$ where $\bs x^*$ is the original image and $\widetilde{\bs x}$ is the solution from \eqref{eq:min1} or \eqref{eq:min4}. We report the RSNR calculated over the entire image and over the RoI. 

In the following, we show a reconstruction example for $20\%$ of the measurements, \ie $m = 0.2N$. Fig.~\ref{fig:SR1} depicts the reconstruction results using the single-resolution approach~\eqref{eq:min1}, and Fig.~\ref{fig:MR1_q2} and Fig.~\ref{fig:MR1_q4} show the reconstruction results using the multi-resolution approach~\eqref{eq:min4} for $L=2$ and $L=4$, respectively. 

The results demonstrate the advantage of the multi-resolution approach~\eqref{eq:min4} over the single-resolution one \eqref{eq:min1} for reconstructing compressible images from highly compressive measurements. Indeed, comparisons of Fig.~\ref{fig:SR1} with Figs.~\ref{fig:MR1_q2} and \ref{fig:MR1_q4} show that the RSNR from multi-resolution reconstruction~\eqref{eq:min4} (either $L=2$ or $L=4$) is more than 6 dB better than the single resolution reconstruction~\eqref{eq:min1}. 

Note that in all the results, the RSNR calculated over the entire image (on the left hand side of Figs.~\ref{fig:SR1}-\ref{fig:MR1_q4}) is higher than the RSNR calculated over the RoI (on the left hand side of Figs.~\ref{fig:SR1}-\ref{fig:MR1_q4}) because within each reconstructed image, the region outside of the RoI has better accuracy than the RoI, which is more difficult to reconstruct. 

It can been observed that the RSNR improvement from $L=2$ to $L=4$ is very small. This observation can be well explained by the result of Prop.~\ref{th:th2}. On one hand, as $L$ increases from $2$ to $4$, condition \eqref{eq:RIPL} is more easily met, and therefore, the quality of the reconstructed image increases. On the other hand, as $L$ increases, the interpolation made by the expansion matrix $\bs E$ becomes more inaccurate, making the term on the right hand side of~\eqref{eq:15} larger and, therefore, the error in the reconstruction increases too. It is likely that, because of this trade-off between the effect of reducing the number of unknowns and the effect of increasing the error in the constraint, the RSNR results from $L=2$ and $L=4$ are very similar.  

\begin{figure}[!t]
\centering
\footnotesize
\begin{tabular}{c c}
	\vspace{-0.5cm} & \\
	\includegraphics[height=2.5cm]{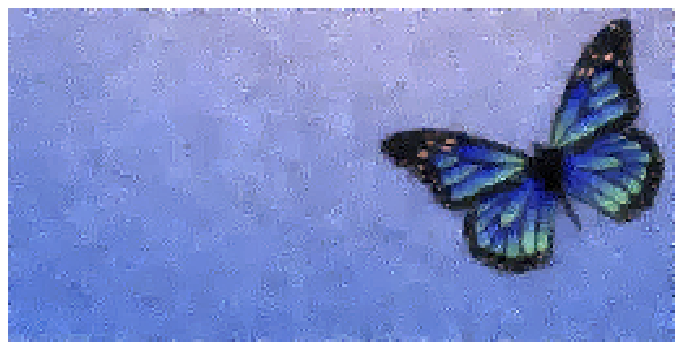} &
	\includegraphics[height=2.5cm]{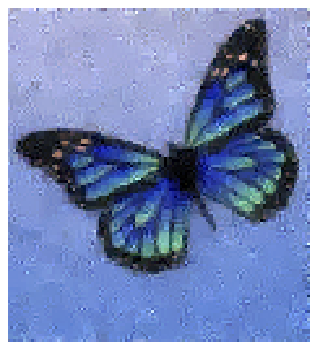} \\
	\vspace{-0.8cm} & \\
	RSNR = 16.95 dB & RSNR = 14.82 dB
\end{tabular}
  \caption{Reconstruction results using the single-resolution approach~\eqref{eq:min1} for $m = 0.2N$. Left: Reconstructed image $\widetilde{\bs x}$. Right: Reconstructed RoI $\widetilde{\bs x}_1$. \\[-.5cm]}
 \label{fig:SR1}
\end{figure}
\begin{figure}[!t]
\centering
\footnotesize
\begin{tabular}{c c}
	\vspace{-0.5cm} & \\
	\includegraphics[height=2.5cm]{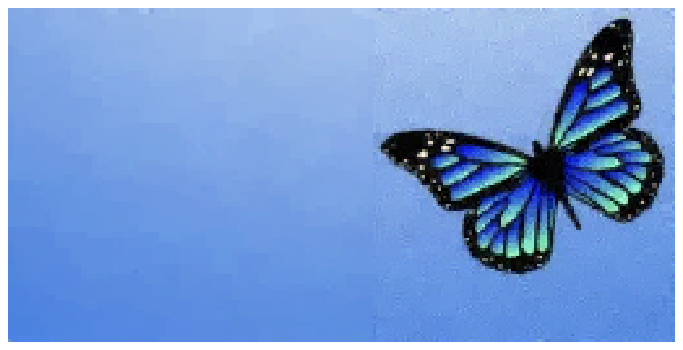} &
	\includegraphics[height=2.5cm]{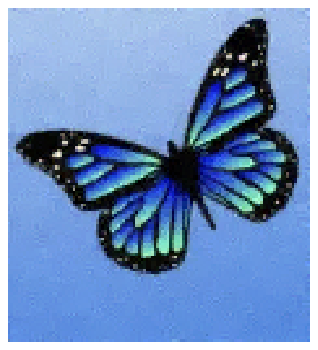} \\
	\vspace{-0.8cm} & \\
	RSNR = 23.23 dB & RSNR = 21.28 dB
\end{tabular}
  \caption{Reconstruction results using the multi-resolution approach~\eqref{eq:min4} for $m = 0.2N$ and $L = 2$. Left: Reconstructed image $\widetilde{\bs x}$. Right: Reconstructed RoI $\widetilde{\bs x}_1$. \\[-.5cm]}
 \label{fig:MR1_q2}
\end{figure}
\begin{figure}[!t]
\centering
\footnotesize
\begin{tabular}{c c}
	\vspace{-0.5cm} & \\
	\includegraphics[height=2.5cm]{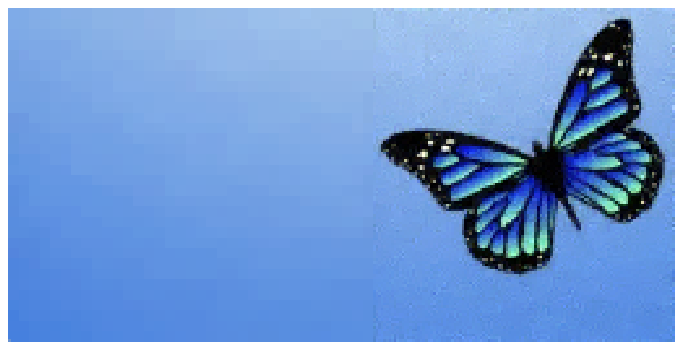} &
	\includegraphics[height=2.5cm]{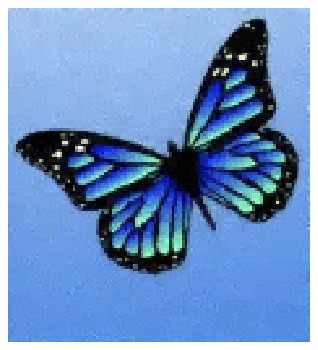} \\
	\vspace{-0.8cm} & \\
	RSNR = 23.77 dB & RSNR = 21.70 dB
\end{tabular}
  \caption{Reconstruction results using the multi-resolution approach~\eqref{eq:min4} for $m = 0.2N$ and $L = 4$. Left: Reconstructed image $\widetilde{\bs x}$. Right: Reconstructed RoI $\widetilde{\bs x}_1$. \\[-1cm]}
 \label{fig:MR1_q4}
\end{figure}

\sq
\sq
\section{Conclusion}
\label{sec:conc}
\sq
We have provided a theoretical analysis supported by simulation results to show that the multi-resolution approach is effective to reconstruct an image with high quality in a region of interest. Although the analysis here is performed for sparse signals, similar conclusions can be reached for compressible signals and we plan to present the analysis and more simulation results in a future longer version of this paper.


\bibliographystyle{IEEE}
\bibliography{bibfile}

\end{document}